\documentclass[11pt]{article}
\usepackage{mathrsfs}
\usepackage{amsmath}
\usepackage{amsfonts}
\usepackage{amssymb}
\usepackage{latexsym}
\usepackage{xcolor}
\usepackage[dvips]{graphicx}
\usepackage{hyperref}
\newtheorem{theorem}{\bf Theorem}[section]
\newtheorem{lemma}[theorem]{\bf Lemma}
\newtheorem{proposition}[theorem]{\bf Proposition}

\newenvironment{proof}{\noindent{\em Proof:}}{\quad \hfill$\Box$\vspace{2ex}}

\newtheorem{definition}[theorem]{\bf Definition}

\def \bI {\Bbb I}
\def \bN {\Bbb N}

\def \bR {\Bbb R}

\def \bW {{\bf W}}
\def \barW {\bar{\bf W}}
\def \barI {\bar{\bf I}}
\def \barb {\bar{\bf b}}
\def \bbb {{\bf b}}

\def \bP {{\bf P}}

\def \bx {{\bf x}}

\def \ba {{\bf a}}

\def \bc {{\bf c}}
\def \bd {{\bf d}}

\def \bu {{\bf u}}

\def \cD {{\cal D}}
\def \cF {{\cal F}}

\def \cI {{\cal I}}

\def \cN {{\cal N}}

\def \and {\, \mbox{\rm and}\, }

\def \supp {\,{\rm supp}\,}
\def \relu {\,{\rm ReLU}\,}

\def \diag {\,{\rm diag}\,}

\def \ess {{\rm ess}}

\makeatletter

\newcommand{\Rmnum}[1]{\expandafter\@slowromancap\romannumeral #1@}
\makeatother
\begin{document}
\title{\bf Convergence of Deep ReLU Networks}
\author{Yuesheng Xu\thanks{Department of Mathematics \& Statistics, Old Dominion University, Norfolk, VA 23529, USA. E-mail address: {\it y1xu@odu.edu}. Supported in part by US National Science Foundation under grant DMS-1912958 and by Natural Science Foundation of China under grant 11771464. }
\quad and \quad Haizhang Zhang\thanks{School of Mathematics (Zhuhai), Sun Yat-sen University, Zhuhai, P.R. China. E-mail
address: {\it zhhaizh2@sysu.edu.cn}. Supported in part by National Natural Science Foundation of China under grant 11971490, and by Natural Science Foundation of Guangdong Province under grant 2018A030313841. Corresponding author.}}
\date{}
\maketitle

%
%

\begin{abstract}
We explore convergence of deep neural networks with the popular ReLU activation function, as the depth of the networks tends to infinity. To this end, we introduce the notion of activation domains and activation matrices of a ReLU network. By replacing applications of the ReLU activation function by multiplications with activation matrices on activation domains, we obtain an explicit expression of the ReLU network. We then identify the convergence of the ReLU networks as convergence of a class of infinite products of matrices. Sufficient and necessary conditions for convergence of these infinite products of matrices are studied. As a result, we establish necessary conditions for ReLU networks to converge that the sequence of weight matrices converges to the identity matrix and the sequence of the bias vectors converges to zero as the depth of ReLU networks increases to infinity. Moreover, we obtain sufficient conditions in terms of the weight matrices and bias vectors at hidden layers for pointwise convergence of deep ReLU networks.  These results provide mathematical insights to the design strategy of the well-known deep residual networks in image classification.

\medskip


\end{abstract}

\section{Introduction}

Deep neural networks (DNNs) have achieved great successes for a wide range of machine learning problems including face recognition, speech recognition, game intelligence, natural language processing, and autonomous navigation. It is generally agreed that four ingredients contribute to the successes. The first two of them are the availability of vast amounts of training data, and recent dramatic improvements in computing and storage power. The third one is a class of efficient numerical algorithms such as the back-propagation training algorithm via the Stochastic Gradient Decent (SGD), Adaptive Boosting  (AdaBoost) algorithms, and the Expectation-Maximization algorithm (EM). The fourth ingredient, also the most important one, is powerful neural network architectures, such as Convolutional Neural Networks (CNN), Long-Short Time Memory (LSTM) networks, Recurrent Neural Networks (RNN), Generative Adversarial Networks (GAN), Deep Belief Networks (DBN), and Residual Networks (ResNet), which provide a superior way of representing data. We refer to a survey paper \cite{LeCun} and monograph \cite{Goodfellow} for an in-depth overview of deep learning.

Compared to the vast development in engineering and applications, research on the mathematical theory of DNNs is still at its infancy, and yet is undergoing rapid progress. Many interesting papers on the approximation and expressive powers of DNNs have appeared in the past several years. We provide here a brief review. More details can be found in two recent surveys \cite{Devore1,Elbrachter}.
Poggio, Mhaskar, Rosasco, Miranda, and Liao \cite{Poggio} proved that DNNs approximate a class of functions with special compositional structure exponentially better than shallow networks. Montanelli and Du \cite{Montanelli1} and Yarotsky \cite{Yarotsky} estimated the number of parameters needed for DNNs to achieve a certain error tolerance in approximating functions in the Koborov space space and differential functions, respectively. Montanelli and Yang \cite{Montanelli2} achieved error bounds for deep ReLU networks approximation of multivariate functions using the Kolmogorov-Arnold superposition theorem. These three pieces of work indicated that DNNs are able to lessen the curse of dimensionality. E and Wang proved that for analytic functions in a low dimension, the convergence rate of the DNN approximation is exponential. Zhou \cite{Zhou1} established the universality of deep convolutional neural networks. Shen, Yang, and Zhang put forward in a series of works \cite{Shen1,Shen2,Shen3} optimal approximation rates for ReLU networks in terms of width and depth to approximate an arbitrary continuous or H\"{o}lder continuous function. Daubechies, DeVore, Foucart,  Hanin, and Petrova \cite{Daubechies} showed that DNNs possess greater approximation power than traditional methods of nonlinear approximation such as variable knot splines and $n$-term approximation from dictionaries. Wang \cite{Wang} presented a mathematical introduction to generative adversarial nets. Lipschitz and proximal properties of neural networks were investigated in \cite{Combettes,Hasannasab,Scaman,Zou}.

In this paper, we study convergence of deep ReLU networks from a different perspective. We are interested in knowing whether or not a deep ReLU network with a fixed width and an increasing depth  will converge to a meaningful function (as a function of the input variable), as its depth tends to infinity. It is well-known that in linear approximations (for example, Fourier analysis \cite{Stein} and wavelet analysis \cite{I.Daubechies}), issues regarding convergence of an expansion such as Fourier expansion and wavelet expansion are fundamental. In particular, in classical analysis, convergence of Fourier expansions with given coefficients is a basic issue. As deep neural networks are used more and more in approximation as a function class, convergence of a sequence of neural networks approximating to a function has become a pressing and interesting issue.  Along this line, the first question is: What requirements should we impose to the weight matrices and the bias vectors to guarantee that the related ReLU deep neural network will converge to a meaningful function as its number of layers tends to infinity? This paper attempts to answer this question. The neural networks to be considered here are not tied with a specific target function. Convergence of neural networks that result from approximation of a given function will be investigated in a different occasion.

It has long been understood that a neural network with the ReLU activation function results in a piecewise linear function. The first novelty of this work is to identify a subdomain that corresponds to a linear component of the ReLU network as an {\it activation domain} and to an {\it activation matrix} which is a diagonal matrix whose diagonal entries are either 1 or 0. The identification allows us to replace the applications of the ReLU activation function, often a source of technical difficulties, by multiplications with the activation matrices. Making use of this observation, we put forward a useful representation for deep ReLU networks, by which we formulate the convergence of deep ReLU networks as convergence of a class of infinite products of matrices. Necessary and sufficient conditions for convergence of such infinite products of matrices are then established. Based on this understanding, we provide necessary conditions and rather week sufficient conditions for a deep ReLU network to converge. The necessary conditions provide mathematical guidelines for further development of deep ReLU networks. Moreover, the sufficient conditions enable us to interpret the design strategy of the well-known deep residual networks, which have been widely used in image classification, with an insightful mathematical explanation.

{}{We would like to point out the differences of our mathematical study on convergence of neural networks to those studies in the computer science community on convergence of Conjugate Kernels and Neural Tangent Kernels with infinite-depth or infinite-width networks (see, for example, \cite{Cho2009, Daniely2016, Hanin,Hu2021,Huang2020Dynamics, Huang2021OnTN,Jacot,Nguyen,Poole2016, Schoenholz2016,Williams1996}). Conjugate kernels and neural tangent kernels are introduced to help understand the training process of DNNs
and mechanism of generalization. The key idea is to view these kernels as random variables on random initialization of DNNs and to study their limiting probability distributions as the width or depth of the DNN increases to infinity. This is done by fixing all parameters in the DNNs except those in the last layer, and then consider the output vector at the last layer as a random feature vector. The conjugate kernel and neural tangent kernel are then defined as inner product of this output vector and its derivatives, respectively. Our investigation and those on conjugate kernels and neural tangent kernels via random initializations of DNNs belong to different convergence problems of DNNs. We shall study the convergence of the function determined by a DNN as its depth tends to infinity, while those studies on conjugate kernels and neural tangent kernels focus on kernels determined by the last layer only. Moreover, our concerns are on conditions on all parameters (not just those on the last layer) of the network to ensure such a convergence. Finally, our study is on the convergence of a DNN to a deterministic function by analysis methods, while those on conjugate kernels and neural tangent kernels are on convergence of kernels in the probability sense.}

The rest of this paper is organized as follows. In Section 2, we review the definition and notation of neural networks and define the notion of convergence of neural networks when new layers are paved to the existing network so that the depth is increasing to infinity. In Section 3, we introduce the notions of the activation domain and activation matrix, with which we present an explicit expression for deep ReLU networks. Based on this expression, we connect convergence of deep ReLU networks with the existence of two limits involving infinite products of matrices. Conditions for convergence of such infinite products of matrices are examined in Section 4. Finally, in Section 5 we revisit the convergence of deep ReLU networks by presenting sufficient conditions for the pointwise convergence of the deep ReLU networks. Moreover, as an application of the result established, we provide a mathematical interpretation to the design of the successful deep residual networks.

\section{Deep Neural Networks and Convergence}
\setcounter{equation}{0}
In this section, we recall the definition of the deep neural network and formulate its convergence problem to be studied in this paper.

We consider general fully connected feed-forward neural networks with fixed width $m$ and increasing depth $n$, for $m, n\in\bN$, from input domain $[0,1]^d\subseteq\bR^d$ to the output space $\bR^{d'}$. For each $i$ with $1\le i\le n$, let $\bW_i$ and $\bbb_i$ denote respectively the weight matrix and bias vector of the $i$-th hidden layer. That is, $\bbb_i\in\bR^m$ for $1\le i\le n$, $\bW_1\in \bR^{m\times d}$, and $\bW_i\in\bR^{m\times m}$ for $2\le i\le n$. The weight matrix $\bW_{o}$ and bias vector $\bbb_{o}$ of the output layer satisfy $\bW_{o}\in \bR^{d'\times m}$ and $\bbb_{o}\in\bR^{d'}$. The structure of such a deep neural network is determined after the choice of an activation function.


Widely-used activation functions in neural networks include the {\bf ReLU} function
$$
\relu(x):=\max(x,0),\ \ x\in\bR
$$
and the {logistic sigmoid} function
$$
S(x):=\frac{1}{1+e^{-x}},\ \ x\in\bR.
$$
After an activation function $\sigma$ is chosen, the structure of the resulting deep neural network may be illustrated as follows:
\begin{equation}\label{neuralnetworks}
\begin{aligned}
 x\in[0,1]^d\ &  \xrightarrow[\sigma]{\bW_1,\bbb_1} \  x^{(1)} &\xrightarrow[\sigma]{\bW_2,\bbb_2} \  x^{(2)} &\rightarrow\cdots\rightarrow&\xrightarrow[\sigma]{\bW_n,\bbb_n}  x^{(n)}&\xrightarrow{\bW_{o},\bbb_{o}}\  y\in\bR^{d'}. \\
\mbox{input}\quad& \quad\mbox{1st layer}&\mbox{ 2nd layer} & & \mbox{$n$-th layer}&\quad\mbox{ output}
\end{aligned}
\end{equation}
Here
\begin{equation}\label{neuralnetworks-termk}
x^{(k)}:=\sigma(\bW_{k}x^{(k-1)}+\bbb_k),\ \ 1\le k\le n \ \mbox{ with } \ x^{(0)}=x,
\end{equation}
\begin{equation}\label{neuralnetworks-output}
y:=\bW_{o}x^{(n)}+b_{o},
\end{equation}
and the activation function $\sigma$ is applied to a vector componentwise. Thus, the above deep neural network determines a continuous function $x\to y$ from $[0,1]^d$ to $\bR^{d'}$.

Consecutive compositions of functions are typical operations used in deep neural networks.
To have a compact form, below we define the notation for consecutive compositions of functions.

\begin{definition} {\bf (Consecutive composition)} \label{Consecutive-compostion}
Let $f_1, f_2, \dots,f_n$ be a finite sequence of functions such that the range of $f_i$ is contained in the domain of $f_{i+1}$, $1\le i\le n-1$, the consecutive composition of $\{f_i\}_{i=1}^n$ is defined to be function
$$
\bigodot_{i=1}^n f_i:=f_n\circ f_{n-1}\circ\cdots\circ f_2\circ f_1,
$$
whose domain is that of $f_1$.
\end{definition}

Note that whenever the consecutive composition notation is used, the order of compositions given in Definition \ref{Consecutive-compostion} is always assumed.

Using the notation defined in Definition \ref{Consecutive-compostion} for consecutive compositions of functions, equations \eqref{neuralnetworks-termk} and \eqref{neuralnetworks-output} may be rewritten as
$$
x^{(k)}=\left(\bigodot_{i=1}^k \sigma(\bW_i \cdot+\bbb_i)\right)(x),\ \ 1\le k\le n
$$
and
$$
y=\bW_{o}\left(\bigodot_{i=1}^n \sigma(\bW_i \cdot+\bbb_i)\right)(x)+b_{o},\ \ x\in[0,1]^d,
$$
respectively.
We are concerned with convergence of the above functions determined by the deep neural network as $n$ increases to infinity. One sees that the output layer is a linear function of $x^{(n)}$ and thus, it does not affect the convergence. By this observation, we introduce the following definition.

\begin{definition} {\bf (Convergence of neural networks)} Let $\bW:=\{\bW_n\}_{n=1}^\infty$ with $\bW_1\in\bR^{m\times d}$, $\bW_n\in\bR^{m\times m}$, $n\ge 2$ be a sequence of weight matrices, and $\bbb:=\{\bbb_n\}_{n=1}^\infty$ with $\bbb_n\in\bR^m$ be a sequence of bias vectors. Define the deep neural network by
$$
\cN_n(x):=\left(\bigodot_{i=1}^n \sigma(\bW_i \cdot+\bbb_i)\right)(x),\ \ x\in[0,1]^d.
$$
We say the deep neural network $\cN_n$ determined by $\bW$, $\bbb$ and a chosen activation function $\sigma$ converges with respect to some norm $\|\cdot\|$ to a limit function $\cN$ if
$$
\lim_{n\to\infty} \|\cN_n-\cN\|=0.
$$
\end{definition}

The goal of this paper is to understand what conditions are required for the weight matrices $\bW_n$ and the bias vectors $\bbb_n$ to ensure convergence of the deep neural network when the activation function is chosen to be ReLU.

\section{Convergence of ReLU Networks}
\setcounter{equation}{0}

In this section, we consider convergence of a deep ReLU network $\cN_n$ as the number $n$ of layers goes to infinity. For this purpose, we introduce an algebraic formulation of a deep ReLU network convenient for convergence analysis.

It has been understood \cite{Devore1} that the neural network (\ref{neuralnetworks}) with $\sigma$ being the ReLU activation function determines a function
$$
f_n(x)=W_o\cN_n(x)+b_o,\ \ x\in[0,1]^d,
$$
that is piecewise linear. Our novelty is to identify the linear components of $\cN_n$ and their associated subdomains by using a sequence of activation matrices.

We begin with analyzing a one layer ReLU network $\cN_1$, which has the form
$$
\cN_1(x):=\sigma (\bW_1x+\bbb_1), \ \ x\in [0,1]^d,
$$
where $\sigma$ is the ReLU activation function. Note that the $m$ components of $\bW_1x+\bbb_1$ are linear functions $\ell_j(x)$, $x\in [0,1]^d$ for $j=1,2, \dots, m$. Hence,
\begin{equation}\label{reluononelayer}
    \cN_1(x)=[\sigma(\ell_j(x)): j=1,2,\dots, m]^T.
\end{equation}
According to the definition of the ReLU activation function, we observe for $j=1,2,\dots, m$ that
\begin{equation}\label{rangeofsigmafunction}
\sigma(\ell_j(x))=0, \ \ \mbox{if}\ \ \ell_j(x)\leq 0, \ \ \mbox{and}\ \
\sigma(\ell_j(x))=\ell_j(x), \ \ \mbox{if}\ \ \ell_j(x)> 0.
\end{equation}
When $\sigma(\ell_j(x))=0$, we say that the node with $\ell_j(x)$ is {\bf deactivated}, and otherwise, it is {\bf activated}. Apparently, there are at most $2^m$ different activation patterns at the first layer. To describe these patterns, we introduce a set of $m\times m$ diagonal matrices whose diagonal entries are either 1 or 0.

Specifically, we define the set of {\bf activation matrices} by
$$
\cD_m:=\left\{\diag(a_1,a_2,\dots,a_m):a_i\in\{0,1\},1\le i\le m\right\}.
$$
An element of $\cD_m$ is either the identity matrix or its degenerated matrix (some diagonal entries degenerated to zero).
The support of an activation matrix $J\in \cD_m$ is defined by
$$
\supp J:=\{k:\ J_{kk}=1,\ \ 1\le k\le m\}.
$$
Clearly, an activation matrix $J\in \cD_m$ is uniquely determined by its support. The set $\cD_m$ of the activation matrices has exactly $2^m$ elements since each of the $m$ diagonal entries of an element in the set has exactly two different choices. This matches the number of possible different activation patterns of a ReLU neural network:
Each element of the set $\cD_m$ corresponds to an activation pattern. For this reason, it is convenient to use $\cD_m$ as an index set.

\begin{definition} {\bf (Activation domains of one layer network)} \label{ActDomofOne}
For a weight matrix $\bW$ with $m$ rows and a bias vector $\bbb\in\bR^m$, the activation domain of $\sigma(\bW x+\bbb)$ with respect to a diagonal matrix $J\in\cD_m$ is
$$
D_{J,\bW,\bbb}:=\left\{x\in\bR^{m'}: (\bW x+\bbb)_j>0\ \mbox{ for }j\in\supp J\mbox{ and }(\bW x+\bbb)_j\le0\mbox{ for }j\notin\supp J\right\}.
$$
\end{definition}

Note that the integer $m'$ in Definition \ref{ActDomofOne} may be chosen to be $d$ (when it is used to define activation domains of the first layer) or $m$ (when it is used to define activation domains of layers which are not the first layer).

In Definition \ref{ActDomofOne}, we use an activation matrix $J\in\cD_m$ to associate an activation pattern of the $m$ components of $\bW x+\bbb$. As a result, Definition \ref{ActDomofOne} enables us to construct a partition of the unit cube $[0,1]^d$ that corresponds to the piecewise linear nature of the function $\cN_1$ and allows us to reexpress $\cN_1$ in a piecewise linear manner. Specifically, we have that
\begin{equation}\label{partition}
[0,1]^d=\bigcup_{I_1\in\cD_m}(D_{I_1,\bW_1, \bbb_1}\cap [0,1]^d).
\end{equation}
By equations (\ref{reluononelayer}) and \eqref{partition}, the one layer ReLU network can be reexpressed as
\begin{equation}\label{expressionofN1}
\cN_1(x)=I_1(\bW_1x+\bbb_1), \ \ x\in D_{I_1,\bW_1, \bbb_1}, \ \ \mbox{for}\ \ I_1\in \cD_m.
\end{equation}
Clearly, on each activation domain $D_{I_1,\bW_1, \bbb_1}$, $\cN_1$ is a linear function.
The essence of equation \eqref{expressionofN1} is that we are able to replace the application of the ReLU activation function by multiplication with an activation matrix in $\cD_m$. This will lead to great convenience in processing ReLU networks. We remark that some of the $2^m$ activation domains might be empty. In fact, by \cite{Zaslavsky}, the number of activation domains with nonempty interior does not exceed
$$
\sum_{k=0}^d \binom{m}{k}.
$$

For a deep ReLU neural network with $n$ layers, we need a sequence of $n$ activation matrices $\barI_n:=(I_1,I_2,\dots,I_n)\in (\cD_m)^{n}$ to identify its different activation patterns on the $n$ hidden layers, where $I_k$ marks the activation pattern at the $k$-th layer. We next define activation domains of a multi-layer network.

\begin{definition}\label{multiactivationdomain} {\bf (Activation domains of a multi-layer network)}
For $\barW_n:=(\bW_1,\dots,\bW_n)\in\bR^{m\times d}\times(\bR^{m\times m})^{n-1},\ \barb_n:=(\bbb_1,\dots,\bbb_n)\in(\bR^m)^n$, the activation domain of
$$
\bigodot_{i=1}^n\sigma(\bW_i\cdot+\bbb_i)
$$
with respect to $\barI_n:=(I_1,\dots,I_n)\in (\cD_m)^{n}$ is defined recursively by
$$
D_{\barI_1,\barW_1,\barb_1}= D_{I_1,\bW_1,\bbb_1}\cap[0,1]^d
$$
and
$$
D_{\barI_n,\barW_n,\barb_n}=\left\{x\in D_{\barI_{n-1},\barW_{n-1},\barb_{n-1}}:\ \biggl(\bigodot_{i=1}^{n-1}\sigma(\bW_i\cdot+\bbb_i)\biggr)(x)\in D_{I_n,\bW_n,\bbb_n}\right\}.
$$
\end{definition}

We have the following observation regarding the activation domain.

\begin{proposition} The sequence of the activation domains $D_{\barI_n,\barW_n,\barb_n}$ are nested, that is,
$$
D_{\barI_{n+1},\barW_{n+1},\barb_{n+1}}\subseteq D_{\barI_n,\barW_n,\barb_n},\ \ n\in\bN.
$$
Moreover, for each $n\in \bN$,
\begin{equation}\label{equivalentactivation}
D_{\barI_n,\barW_n,\barb_n}=\left\{x\in D_{I_1,\bW_1,\bbb_1}\cap[0,1]^d:\ \biggl(\bigodot_{i=1}^{k-1}\sigma(\bW_i\cdot+\bbb_i)\biggr)(x)\in D_{I_k,\bW_k,\bbb_k},\ 2\le k\le n\right\}.
\end{equation}
\end{proposition}
\begin{proof}
That $D_{\barI_n,\barW_n,\barb_n}$ are nested follows directly from the definition. Equality (\ref{equivalentactivation}) may be proved by induction on $n$.
\end{proof}

The activation domain $D_{\barI_n,\barW_n,\barb_n}$ characterizes the inputs $x\in[0,1]^d$ such that when those inputs are going through the ReLU neural network (\ref{neuralnetworks}), at the $k$-th hidden layer ($1\le k\le n$), exactly the nodes with index in $\supp I_k$ are activated. There are at most $2^{nm}$ activation domains $D_{\barI_n,\barW_n,\barb_n}$ corresponding to all choices of sequences of diagonal matrices $\barI_n\in(\cD_m)^{n}$, and a large number of them might be empty or have zero Lebesgue measure. {We remark that it is well-known that a ReLU network outputs a function that is piece-wise linear. The number of linear regions of a ReLU network is estimated in \cite{Montufar,Serra}. Our novelty here is to write down an explicit expression of the ReLU network using the activation domains.}

For each positive integer $n$, the activation domains
$$
D_{\barI_n,\barW_n,\barb_n},\ \  \mbox{for}\ \ \barI_n:=(I_1,\dots,I_n)\in (\cD_m)^{n},
$$
form a partition of the unit cube $[0,1]^d$. That is,
for each $n\in \bN$,
$$
[0,1]^d=\bigcup_{\barI_n\in (\cD_m)^{n}}D_{\barI_n,\barW_n,\barb_n}
$$
By using these activation domains, we are able to write down an explicit expression of the ReLU network $\cN_n$ with applications of the ReLU activation function replaced by multiplications with the activation matrices.

We now establish a representation of the ReLU network based on its activation domains and activation matrices. To this end,
we need a notation to denote the product of matrices with a prescribed order. Specifically, we write
$$
\prod_{i=1}^n\bW_i=\bW_n\bW_{n-1}\cdots\bW_1.
$$
For $n,k\in \bN$, we also adopt the following convention that
$$
\prod_{i=k}^n\bW_i=\bW_n\bW_{n-1}\cdots\bW_k, \ \ \mbox{for} \ \  n\geq k, \ \ \mbox{and}\ \
\prod_{i=k}^n\bW_i={ I}, \ \ \mbox{for} \ \  n< k,
$$
where ${ I}$ denotes the $m\times m$ identity matrix.

\begin{theorem}\label{expressionofcnprop}
It holds that
\begin{equation}\label{compositions2}
\cN_n(x)=\left(\prod_{i=1}^nI_i\bW_i\right)x+\sum_{i=1}^n\left(\prod_{j=i+1}^n I_j\bW_j\right)I_i\bbb_i,\ \ x\in D_{\barI_n,\barW_n,\barb_n},  \ \ \barI_n:=(I_1,\dots,I_n)\in (\cD_m)^{n}.
\end{equation}
\end{theorem}
\begin{proof}
We prove by induction on $n$. When $n=1$, by (\ref{expressionofN1}), the result is true. Suppose that (\ref{compositions2}) holds for $n-1$. Now let $x\in D_{\barI_n,\barW_n,\barb_n}$. Then
$$
\cN_n(x)=\sigma\biggl(\bW_n \biggl(\bigodot_{i=1}^{n-1} \sigma(\bW_i\cdot+\bbb_i)\biggr)(x) +\bbb_n\biggr).
$$
By Definition \ref{multiactivationdomain},
$$
\biggl(\bigodot_{i=1}^{n-1} \sigma(\bW_i\cdot+\bbb_i)\biggr)(x)\in D_{I_n,\bW_n,\bbb_n}.
$$
We then get by (\ref{expressionofN1}) and induction that
$$
\begin{aligned}
\cN_n(x)&=I_n\left(\bW_n \left(\bigodot_{i=1}^{n-1} \sigma\left(\bW_i\cdot+\bbb_i\right)\right)(x) +\bbb_n\right)\\
&=I_n\bW_n\left(\left(\prod_{i=1}^{n-1}I_i\bW_i\right)x+\sum_{i=1}^{n-1}\left(\prod_{j=i+1}^{n-1} I_j\bW_j\right)I_i\bbb_i\right)+I_n\bbb_n\\
&=\left(\prod_{i=1}^nI_i\bW_i\right)x+\sum_{i=1}^n\left(\prod_{j=i+1}^n I_j\bW_j\right)I_i\bbb_i,
\end{aligned}
$$
which proves (\ref{compositions2}).
\end{proof}

The representation of a deep ReLU network established in Theorem \ref{expressionofcnprop} is crucial for further investigation of the network. The piecewise linear property of a ReLU network follows immediately from this representation. It is also helpful for developing the convergence results of ReLU Networks later in this paper.

In the remaining part of this section, we formulate the convergence of deep ReLU networks as a problem about convergence of infinite products of matrices.
Denote by $\|\cdot\|_p$ the $\ell^p$-norm on $\bR^m$, $1\le p\le+\infty$. For a Lebesgue measurable subset $\Omega\subseteq\bR^d$, by $L^p(\Omega,\bR^m)$ we denote the space of all real-valued functions $f:\Omega\to\bR^m$ such that each component of $f$ is Lebesgue measurable on $\Omega$ and such that
$$
\|f\|_p:=\left\{
\begin{array}{ll}
\displaystyle{\biggl(\int_{\Omega}\|f(x)\|_p^pdx\biggr)^p},&1\le p<+\infty,\\
\displaystyle{\ess\sup_{x\in\Omega} \|f(x)\|_\infty},&p=+\infty
\end{array}
\right.
$$
is finite. Also, $C(\Omega,\bR^m)$ is the space of all continuous functions from $\Omega$ to $\bR^m$.

Theorem \ref{expressionofcnprop} allows us to present a necessary and sufficient condition for ReLU neural networks $\cN_n$ to converge to a function in $L^p([0,1]^d,\bR^m)$, as $n\to\infty$.
Let $\bW:=\{\bW_n\}_{n=1}^\infty$ with $\bW_1\in \bR^{m\times d}$, $\bW_n\in \bR^{m\times m}$, $n>1$ and $\bbb:=\{\bbb_n\}_{n=1}^\infty$ with $\bbb_n\in\bR^m$ be a sequence of weight matrices and bias vectors, respectively. Suppose $\cN\in C([0,1]^d,\bR^m)$.
It follows from Theorem \ref{expressionofcnprop} that the ReLU neural networks $\cN_n$ converge to $\cN$ in $L^p([0,1]^d,\bR^m)$ if and only if
\begin{equation}\label{convergenceidentity}
    \lim_{n\to\infty}\sum_{\barI_n\in (\cD_m)^{n}}\int_{D_{\barI_n,\barW_n,\barb_n}}\left\|
    \left(\prod_{i=1}^nI_i\bW_i\right)x+\sum_{i=1}^n\left(\prod_{j=i+1}^n I_j\bW_j\right)I_i\bbb_i-\cN(x)
    \right\|_p^pdx=0,\ \ 1\leq p<+\infty
\end{equation}
and
\begin{equation}\label{convergenceidentityp=infinity}
    \lim_{n\to\infty}\max_{\barI_n\in (\cD_m)^{n}}\sup_{x\in D_{\barI_n,\barW_n,\barb_n}}\left\|
    \left(\prod_{i=1}^nI_i\bW_i\right)x+\sum_{i=1}^n\left(\prod_{j=i+1}^n I_j\bW_j\right)I_i\bbb_i-\cN(x)
    \right\|_{\infty}=0,\ \ p=\infty.
\end{equation}

This necessary and sufficient condition together with Theorem \ref{expressionofcnprop} leads to useful necessary conditions and sufficient conditions for the sequence of ReLU neural networks to converge. They will be presented next. To this end, we first establish a technical lemma.

\begin{lemma}\label{convergenceoflinearfunctions}
Let $A_n\in\bR^{m\times d},\ b_n\in\bR^m$, $n\in\bN$ and let $1\le p\le+\infty$. Then the sequence of linear functions
$$
A_nx+b_n
$$
converges in $L^p(\Omega,\bR^m)$ on a bounded subset $\Omega\subseteq\bR^d$ that has positive Lebesgue measure if and only if both $\{A_n\}$ and $\{b_n\}$ converge.
\end{lemma}
\begin{proof}
We first prove the sufficient condition.
Suppose that both $\{A_n\}$ and $\{b_n\}$ converge. Then, clearly, $A_nx+b_n$ converges uniformly on $\Omega$ as $\Omega$ is bounded. As a result, $\{A_nx+b_n\}$ converges in $L^p(\Omega,\bR^m)$ for all $1\le p\le+\infty$.

Conversely, suppose that $\{A_nx+b_n\}$ converges to some limit function $\bu:=(u_1,u_2,\dots,u_m)^T$ in $L^p(\Omega,\bR^m)$ for some
$p\in[1,+\infty]$, where $\Omega$ has positive Lebesgue measure. Let $b_n:=(b_{n1},b_{n2},\dots,b_{nm})^T$ and $A_n:=[A_{n,jk}:1\le j\le m,1\le k\le d]$. Thus, for each $j$ with $1\le j\le m$, we have that
\begin{equation}\label{convergenceoflinearfunctionseq1}
(A_nx+b_n)_j=b_{nj}+\sum_{k=1}^d A_{n,jk}x_k\to u_j\mbox{ in }L^p(\Omega)\mbox{ as }n\to\infty.
\end{equation}
As $\Omega$ has a positive measure, $C(\Omega)$ is infinite-dimensional. Therefore, there exists $g\in C(\Omega)$ such that
$$
\int_\Omega g(x)x_kdx=0,\ \ \mbox{for all}\ \  1\le k\le d \mbox{ and } \int_\Omega g(x)dx=1.
$$
Equation (\ref{convergenceoflinearfunctionseq1}) ensures that
$$
\lim_{n\to\infty}b_{nj}=\lim_{n\to\infty}\int_\Omega g(x)(A_nx+b_n)_jdx=\int_\Omega g(x)u_j(x)dx,
$$
which implies that for every $1\le j\le m$, $b_{nj}$ converges as $n\to\infty$. Hence, $\{b_n\}$ converges. Similarly, for each $l$ with $1\le l\le d$, there exists a function $h_l\in C(\Omega)$ such that
$$
\int_\Omega h_l(x)x_kdx=\delta_{l,k},\ \ \mbox{for all}\ \ 1\le k\le d \ \mbox{ and }\ \int_\Omega h_l(x)dx=0.
$$
Again, by  (\ref{convergenceoflinearfunctionseq1}), we have that
$$
\lim_{n\to\infty}A_{n,jl}=\lim_{n\to\infty}\int_\Omega h_l(x)(A_nx+b_n)_jdx=\int_\Omega h_l(x)u_j(x)dx,
$$
which proves the convergence of $\{A_n\}$.
\end{proof}

We are now ready to prove the main result of this section.

\begin{theorem}\label{convergenceRELU}
Let $\bW:=\{\bW_n\}_{n=1}^\infty$ with $\bW_1\in \bR^{m\times d}$, $\bW_n\in \bR^{m\times m}$, $n>1$ and $\bbb:=\{\bbb_n\}_{n=1}^\infty$ with $\bbb_n\in\bR^m$ be a sequence of weight matrices and bias vectors, respectively.

\begin{enumerate}
\item {\bf (Necessary condition for convergence)} If the sequence of ReLU networks $\{\cN_n\}_{n=1}^\infty$ converges in $L^p([0,1]^d,\bR^m)$ then for all sequences of diagonal matrices $\bI=(I_n\in\cD_m:\ n\in\bN)$ such that the set
$$
\bigcap_{n=1}^\infty D_{\barI_n,\barW_n,\barb_n}
$$
has positive Lebesgue measure, the two limits
\begin{equation}\label{limit1}
\prod_{n=1}^\infty I_n\bW_n:=\lim_{n\to\infty}\prod_{i=1}^nI_i\bW_i
\end{equation}
and
\begin{equation}\label{limit2}
\lim_{n\to\infty}\sum_{i=1}^n\left(\prod_{j=i+1}^n I_j\bW_j\right)I_i\bbb_i
\end{equation}both exist.

\item {\bf (Sufficient condition for pointwise convergence)} If for all sequences of diagonal matrices $\bI=(I_n\in\cD_m:\ n\in\bN)$, the above two limits both exist, then the sequence of ReLU neural networks $\{\cN_n\}_{n=1}^\infty$ converges pointwise on $[0,1]^d$.
\end{enumerate}
\end{theorem}

\begin{proof}
We prove the first claim of this theorem. If for a sequence of diagonal matrices $\bI=(I_n\in\cD_m:\ n\in\bN)$,
$$
D_\bI:=\bigcap_{n=1}^\infty D_{\barI_n,\barW_n,\barb_n}
$$
has positive Lebesgue measure, then by (\ref{convergenceidentity}) and (\ref{convergenceidentityp=infinity}),
$$
\|\cN_n-\cN\|_{L^p(D_\bI,\bR^m)}\le \|\cN_n-\cN\|_{L^p(D_{\barI_n,\barW_n,\barb_n},\bR^m)}\le\|\cN_n-\cN\|_{L^p([0,1]^d,\bR^m)}\to 0,\ n\to\infty.
$$
It implies
$$
\cN_n(x)=\left(\prod_{i=1}^nI_i\bW_i\right)x+\sum_{i=1}^n\left(\prod_{j=i+1}^n I_j\bW_j\right)I_i\bbb_i
$$
converges in $D_\bI$ with respect to the chosen $L^p$-norm $\|\cdot\|$. The proof may be completed by applying Lemma \ref{convergenceoflinearfunctions} with $\Omega:=D_\bI$.

Next, we establish the second claim.  For every $x\in[0,1]^d$, there exists a sequence of diagonal matrices $\bI=(I_n:\ I_n\in\cD_m,\ n\in\bN)$ such that
$$
x\in\bigcap_{n=1}^\infty D_{\barI_n,\barW_n,\barb_n}.
$$
Thus, by (\ref{compositions2})
$$
\cN_n(x)=\left(\prod_{i=1}^nI_i\bW_i\right)x+\sum_{i=1}^n\left(\prod_{j=i+1}^n I_j\bW_j\right)I_i\bbb_i,\ \ n\in\bN.
$$
Therefore, the existence of the two limits (\ref{limit1}) and (\ref{limit2}) are sufficient for pointwise convergence of $\{\cN_n(x)\}$.
\end{proof}

Theorem \ref{convergenceRELU} lays a foundation for studying the convergence issue of deep ReLU networks.

\section{Infinite Products of Matrices}
\setcounter{equation}{0}

This section is devoted to investigation of convergence of infinite products of matrices, which arise in the study of convergence of deep ReLU networks.

It follows from Theorem \ref{convergenceRELU} that the convergence of ReLU networks is reduced to existence of the two limits (\ref{limit1}) and (\ref{limit2}). Specifically, the convergence of the infinite product of matrices
\begin{equation}\label{infinitematrixproducts}
\prod_{n=1}^\infty I_n\bW_n, \ \ \mbox{for any}\ \ I_n\in \cD_m,
\end{equation}
appears in the two limits. We hence raise the following question:
What conditions on the matrices $\bW_n$, $n\in\bN$, will guarantee the convergence of the infinite product (\ref{infinitematrixproducts}) for all choices $I_n\in\cD_m$, $n\in\bN$? We first answer this question.

There is a well-known sufficient condition (\cite{Wedderburn}, page 127) for convergence of infinite products of matrices, which can be considered as a generalization of the convergence of infinite products $\prod_{n=1}^\infty(1+x_i)$ of scalars. The result states that if
\begin{equation}\label{sufficient}
\bW_n=I+\bP_n\mbox{ and }\sum_{n=1}^\infty\|\bP_n\|<+\infty,
\end{equation}
where $I$ is the identity matrix and $\|\cdot\|$ is any matrix norm satisfying $\|AB\|\le \|A\|\|B\|$, then the infinite product
$$
\prod_{n=1}^\infty \bW_n
$$
converges. This result was extended by Artzrouni \cite{Artzrouni}. Our question differs from those results in having the diagonal matrices $I_n$ in (\ref{infinitematrixproducts}) arbitrarily chosen from $\cD_m$. Nevertheless, we manage to prove that this condition (\ref{sufficient}) remains sufficient for the convergence of (\ref{infinitematrixproducts}). We proceed to establish this result.

Let $\|\cdot\|$ be a norm on $\bR^m$ that is nondecreasing on the modules of vector components:
\begin{equation}\label{nondecreasingvectornorm}
    \|\ba\|\le\|\bbb\| \mbox{ whenever }|a_i|\le |b_i|,\ 1\le i\le m,\ \ \mbox{for}\ \ba=(a_1,a_2,\dots,a_m), \bbb=(b_1,b_2,\dots,b_m)\in\bR^m.
\end{equation}
This requirement on a vector norm is mild and it is satisfied by the $\ell^p$-norms for all $1\le p\le +\infty$. We then define its induced matrix norm on $\bR^{m\times m}$, also denoted by $\|\cdot\|$, by
$$
\|A\|=\sup_{x\in\bR^m,x\ne0}\frac{\|Ax\|}{\|x\|},\ \ \mbox{for}\ \ A\in\bR^{m\times m}.
$$
Clearly, this matrix norm has the property that
\begin{equation}\label{matrixcon1}
\|AB\|\le \|A\|\|B\| \mbox{ for all matrices }A,B
\end{equation}
and
\begin{equation}\label{matrixcon2}
\|I_i\|\le 1 \mbox{ for each }I_i\in \cD_m.
\end{equation}
Note that the Frobenius norm satisfies (\ref{matrixcon1}) but does not satisfy (\ref{matrixcon2}).

Our first observation regards the product of activation matrices.

\begin{lemma}
If $j\ge 2$ and $I_i\in\cD_m$ for $i=j, j+1, \dots, n$, then
$$
\lim_{n\to \infty}\prod_{i=j}^nI_i=\cI_j,
$$
for some matrix $\cI_j\in\cD_m$, and there exist a positive integer $N$ such that
\begin{equation}\label{sufficienteq1}
\prod_{i=j}^nI_i=\cI_j, \ \ \mbox{whenever}\ \ n>N.
\end{equation}
\end{lemma}
\begin{proof}
We first note that for all $n\in\bN$, there holds $\prod_{i=j}^nI_i\in\cD_m$ and
$$
\supp\left(\prod_{i=j}^nI_i\right)=\bigcap_{i=j}^n\supp I_i.
$$
It follows that
\begin{equation}\label{UpperBoundSet}
    \emptyset \subseteq\supp\left(\prod_{i=j}^{n+1}I_i\right)\subseteq \supp\left(\prod_{i=j}^nI_i\right),
\end{equation}
where $\emptyset$ denotes the empty set.
Therefore, the limit
$$
\lim_{n\to\infty }\supp\left(\prod_{i=j}^nI_j\right)=\bigcap_{i=j}^\infty \supp I_i
$$
exists. Note that there exists a unique diagonal matrix $\cI_j\in\cD_m$ such that
$$
\supp \cI_j=\bigcap_{i=j}^\infty \supp I_i
$$
and thus,
$$
\lim_{n\to\infty }\prod_{i=j}^nI_i=\cI_j.
$$

Since the set $\cD_m$ contains only a finite number of matrices, according to \eqref{UpperBoundSet}, there exists a positive integer $N$ such that
$$
\supp\left(\prod_{i=j}^nI_i\right)=\supp\left(\prod_{i=j}^NI_i\right), \ \ \mbox{for all}\ \ n>N.
$$
Thus, there exists  a unique diagonal matrix $\cI_j\in\cD_m$ such that
$$
\cI_j=\prod_{i=j}^NI_i
$$
and
$$
\prod_{i=j}^nI_i=\cI_j, \ \ \mbox{for all}\ \ n>N.
$$
\end{proof}

\begin{lemma}\label{sufficientlemma}
If a sequence $\{a_n\}_{n=1}^\infty$ satisfies $a_n\geq 0$ and $\sum_{n=1}^\infty a_n<+\infty$, then for all $p\in\bN$,
\begin{equation}\label{sufficientlemma2}
\sum_{i=p+1}^\infty a_i+\sum_{l=2}^\infty\sum_{\substack{1\le i_1<i_2<\cdots<i_l\\ i_l>p}}\prod_{k=1}^la_{i_k}\le \left(\sum_{i=p+1}^\infty a_i\right)\exp\left(\sum_{i=1}^\infty a_i\right).
\end{equation}
\end{lemma}
\begin{proof}
Recall the expansion
\begin{equation}\label{exp}
    e^x=\sum_{l=0}^\infty \frac{x^l}{l!}, \ \ \mbox{for all}\ \ x\in \bR.
\end{equation}
Substituting $x:=\sum_{i=1}^\infty a_i$ in equation \eqref{exp} yields that
$$
\exp\left(\sum_{i=1}^\infty a_i\right)=\sum_{l=0}^\infty \frac1{l!}\left(\sum_{i=1}^\infty a_i\right)^l.
$$
Multiplying both sides of the above equation by the sum $\sum_{i=p+1}^\infty a_i$ gives that
\begin{equation}\label{sumexp}
    \left(\sum_{i=p+1}^\infty a_i\right)\exp\left(\sum_{i=1}^\infty a_i\right)=\left(\sum_{i=p+1}^\infty a_i\right)\sum_{l=0}^\infty \frac1{l!}\left(\sum_{i=1}^\infty a_i\right)^l.
\end{equation}
Note that for $l\ge1$
$$
\left(\sum_{i=p+1}^\infty a_i\right)\frac1{(l-1)!}\left(\sum_{i=1}^\infty a_i\right)^{l-1}\ge \sum_{\substack{1\le i_1<i_2<\cdots<i_l\\ i_l>p}}\prod_{k=1}^la_{i_k}.
$$
Combining this inequality with equation \eqref{sumexp} proves the desired inequality of this lemma.
\end{proof}

When the infinite sum in Lemma \ref{sufficientlemma} is reduced to a finite sum, we obtain the following special result that was originally used in \cite{Wedderburn} without a proof.
If $a_i$, $1\le i\le n$, are nonnegative numbers, by setting $a_i=0$ for $i>n$, we then obtain from  Lemma \ref{sufficientlemma} that for all $p<n$,
\begin{equation}\label{sufficientlemma1}
\sum_{i=p+1}^n a_i+\sum_{l=2}^n\sum_{\substack{1\le i_1<i_2<\cdots<i_l\le n\\ i_l>p}}\prod_{k=1}^la_{i_k}\le \left(\sum_{i=p+1}^na_i\right)\exp\left(\sum_{i=1}^na_i\right).
\end{equation}

We next provide a sufficient condition on the matrices which ensures convergence of the infinite product (\ref{infinitematrixproducts}). In \cite{Artzrouni}, it was proved that if
$$
\sum_{i=1}^\infty \|A_i\|<+\infty
$$
and
\begin{equation}\label{artzrounicon}
\|U_i\|=1\mbox{ for all }i\in\bN
\end{equation}
then
\begin{equation}\label{artzrouni}
\prod_{i=1}^\infty (U_i+A_i)
\end{equation}
converges. Our infinite products (\ref{infinitematrixproducts}) differ from (\ref{artzrouni}) in that we have $I_n$'s arbitrarily chosen from $\cD_m$. Also, the assumption (\ref{artzrounicon}) does not apply to our question. We shall make use of the special property (\ref{sufficienteq1}) of $I_n$'s, which making our approach more direct and simple than that in \cite{Artzrouni}.

\begin{theorem}\label{sufficientthm}
Let $\|\cdot\|$ be a matrix norm satisfying (\ref{matrixcon1}) and (\ref{matrixcon2}). If $\bW_1\in\bR^{m\times d}$, $\bW_n\in\bR^{m\times m}$,  $n\ge2$, are matrices satisfying
\begin{equation}\label{sufficient2}
\bW_n=I+\bP_n, n\ge 2,\mbox{ and }\sum_{n=2}^\infty\|\bP_n\|<+\infty,
\end{equation}
then the infinite product (\ref{infinitematrixproducts}) converges for all $I_n\in\cD_m$, $n\in\bN$.
\end{theorem}
\begin{proof}
It suffices to prove that the infinite product of matrices
$$
\prod_{n=2}^\infty I_n\bW_n
$$
converges under the assumed conditions.

We compute that
$$
\prod_{i=2}^n I_i\bW_i=\prod_{i=2}^n (I_i+I_i\bP_i).
$$
Expanding the product on the right hand side of the above equation yields
\begin{equation}\label{sufficienteq3}
\prod_{i=2}^n I_i\bW_i=\prod_{i=2}^{n} I_i
+\sum_{l=1}^{n-1}\sum_{2\le j_1<j_2<\cdots<j_l\le n}\left(\prod_{k=j_l+1}^nI_k\right)\left(\prod_{i=2}^l\left(I_{j_i}\bP_{j_i}\prod_{k=j_{i-1}+1}^{j_i-1}I_{k}\right)\right)  I_{j_1}\bP_{j_1}\left(\prod_{k=2}^{j_1-1}I_k\right).
\end{equation}
According to (\ref{sufficienteq1}), we assume that $n'>n>p$ are large enough integers so that
$$
\prod_{i=j}^nI_i=\prod_{i=j}^{n'}I_i=\cI_j,\ \ 2\le j\le p+1
$$
for some $\cI_j\in \cD_m$. This fact ensures that for $1\le l\le n-1$, the terms in
$$
\sum_{2\le j_1<j_2<\cdots<j_l\le n}\biggl(\prod_{k=j_l+1}^nI_k\biggr)\biggl(\prod_{i=2}^l\biggl(I_{j_i}\bP_{j_i}\prod_{k=j_{i-1}+1}^{j_i-1}I_{k}\biggr)\biggr)  I_{j_1}\bP_{j_1}\biggl(\prod_{k=2}^{j_1-1}I_k\biggr),
$$
which appear in (\ref{sufficienteq3}) and those in
$$
\sum_{2\le j_1<j_2<\cdots<j_l\le n'}\biggl(\prod_{k=j_l+1}^{n'}I_k\biggr)\biggl(\prod_{i=2}^l\biggl(I_{j_i}\bP_{j_i}\prod_{k=j_{i-1}+1}^{j_i-1}I_{k}\biggr)\biggr)  I_{j_1}\bP_{j_1}\biggl(\prod_{k=2}^{j_1-1}I_k\biggr)
$$
which appear in (\ref{sufficienteq3}) with $n$ replaced by $n'$
are identical if $j_l\le p$.

Now, for $n'>n>p$ we consider the difference
$$
\prod_{i=2}^n I_i\bW_i-\prod_{i=2}^{n'} I_i\bW_i
$$
and use  (\ref{sufficienteq3}) with the fact pointed out above so that the identical terms appearing in the two products are canceled. By
applying the matrix norm to the resulting sum, we get by (\ref{matrixcon1}) and (\ref{matrixcon2})
$$
\begin{aligned}
\biggl\|\prod_{i=2}^n I_i\bW_i-\prod_{i=2}^{n'} I_i\bW_i\biggr\|\le &\sum_{j=p+1}^n \|\bP_j\|+\sum_{j=p+1}^{n'} \|\bP_j\|+\sum_{l=2}^{n-1}\sum_{\substack{2\le j_1<j_2<\cdots<j_l\le n\\j_l>p}}\prod_{k=1}^l\|\bP_{j_k}\|\\
&+\sum_{l=2}^{n-1}\sum_{\substack{2\le j_1<j_2<\cdots<j_l\le n'\\j_l>p}}\prod_{k=1}^l\|\bP_{j_k}\|+\sum_{l=n}^{n'-1}\sum_{2\le j_1<j_2<\cdots<j_l\le n'}\prod_{k=1}^l\|\bP_{j_k}\|.
\end{aligned}
$$
Invoking inequality (\ref{sufficientlemma1}) we obtain from the last inequality for large enough positive integers $n'>n$
\begin{equation}\label{difference-in-norm}
\biggl\|\prod_{i=2}^n I_i\bW_i-\prod_{i=2}^{n'} I_i\bW_i\biggr\|\le 2 \biggl(\sum_{i=p+1}^{\infty}\|\bP_i\|\biggr)\exp\biggl(\sum_{i=2}^\infty\|\bP_i\|\biggr).
\end{equation}

Finally, the second inequality of \eqref{sufficient2} ensures that for  $\varepsilon>0$, there exists $p\in\bN$ such that
\begin{equation}\label{sufficienteq2}
\sum_{j=p+1}^\infty \|\bP_j\|<\varepsilon.
\end{equation}
Using estimate (\ref{sufficienteq2}) in the right hand side of \eqref{difference-in-norm} yields
$$
\left\|\prod_{i=2}^n I_i\bW_i-\prod_{i=2}^{n'} I_i\bW_i\right\|\le 2\varepsilon\exp\left(\sum_{i=2}^\infty\|\bP_i\|\right),
$$
which together with the second inequality of condition  (\ref{sufficient2}) proves the convergence of the infinite product (\ref{infinitematrixproducts}).
\end{proof}

We next deal with the second limit (\ref{limit2}). Our first task is to formulate a necessary condition, showing that the linear function $\bW_nx+b_n$ on the $n$-th layer will be close to the identity mapping for sufficiently large $n$.

\begin{theorem}\label{proplimit2}
Let $\|\cdot\|$ be a norm on $\bR^m$ that satisfies (\ref{nondecreasingvectornorm})  and $\|\cdot\|$ be its induced matrix norm. Suppose that the matrices $\bW_n$, $n\ge 2$, satisfy
\begin{equation}\label{sufficient3}
\bW_n=I+\bP_n, n\ge 2, \ \sum_{n=2}^\infty\|\bP_n\|<+\infty,\mbox{ and }\sum_{i=n+1}^\infty \|\bP_i\|=o\left(\frac1n\right),\ n\to\infty,
\end{equation}
and that the vectors $\bbb_n$, $n\in\bN$, are bounded. If the limit (\ref{limit2}) exists for all choices of matrices $I_i\in \cD_m$, $i\in\bN$, then
\begin{equation}\label{proplimit2eq2}
\lim_{n\to\infty}\bW_n=I
\end{equation}
and
\begin{equation}\label{proplimit2eq2*}
\lim_{n\to\infty}\bbb_n=0.
\end{equation}
\end{theorem}
\begin{proof}
The second inequality of condition (\ref{sufficient3}) implies $\|\bP_n\|\to 0$ as $n\to\infty$. Thus, using the first equation of condition  (\ref{sufficient3}), we conclude equation \eqref{proplimit2eq2}.

It remains to prove equation  \eqref{proplimit2eq2*}.
Since the limit (\ref{limit2}) exists for all $I_i\in \cD_m$, we let $I_i=I$ for $i\ge 2$ to get that
\begin{equation}\label{proplimit2eq3}
\lim_{n\to\infty}\sum_{i=1}^n\biggl(\prod_{j=i+1}^n \bW_j\biggr)\bbb_i
\end{equation}
exists. By similar analysis as those in the proof of Theorem \ref{sufficientthm}, we conclude that
\begin{equation}\label{differencen-infty}
\biggl\|\prod_{j=i+1}^n \bW_j-\prod_{j=i+1}^\infty \bW_j\biggr\|\le \biggl(\sum_{j=n+1}^\infty \|\bP_j\|\biggr)\exp\biggl(\sum_{j=i+1}^\infty\|\bP_j\|\biggr)\le C_1\biggl(\sum_{j=n+1}^\infty \|\bP_j\|\biggr),
\end{equation}
where
$$
C_1:=\exp\biggl(\sum_{j=2}^\infty\|\bP_j\|\biggr).
$$
Noting that $\bbb_n$, $n\in\bN$, are bounded, we may let
\begin{equation}\label{sup-bn}
C_2:=\sup_{n\in\bN}\|\bbb_n\|<+\infty.
\end{equation}
Employing \eqref{differencen-infty} and \eqref{sup-bn} yields the estimate
$$
\begin{aligned}
\biggl\|\sum_{i=1}^n\biggl(\prod_{j=i+1}^n \bW_j\biggr)\bbb_i-\sum_{i=1}^n\biggl(\prod_{j=i+1}^\infty \bW_j\biggr)\bbb_i\biggr\|&\le \sum_{i=1}^n\biggl\|\prod_{j=i+1}^n \bW_j-\prod_{j=i+1}^\infty \bW_j\biggr\|\|\bbb_i\|\\
&\le C_1C_2n\biggl(\sum_{j=n+1}^\infty \|\bP_j\|\biggr).
\end{aligned}
$$
By the third condition in (\ref{sufficient3}) and the existence of the limit (\ref{proplimit2eq3}), we observe that
$$
\lim_{n\to\infty}\sum_{i=1}^n\biggl(\prod_{j=i+1}^\infty \bW_j\biggr)\bbb_i
$$
also exists. It follows that
\begin{equation}\label{proplimit2eq4}
\lim_{i\to\infty}\biggl(\prod_{j=i+1}^\infty \bW_j\biggr)\bbb_i=0.
\end{equation}
Notice that
$$
\prod_{j=i+1}^\infty\bW_j-I=\prod_{j=i+1}^\infty(I+\bP_j)-I=\sum_{j=i+1}^\infty \bP_j+\sum_{l=2}^{\infty}\sum_{i+1\le j_1<j_2<\cdots <j_l} \prod_{k=1}^l \bP_{j_k}.
$$
It follows from Lemma \ref{sufficientlemma} with $a_n:=\|\bP_n\|$ that for $i$ big enough,
$$
\biggl\|\prod_{j=i+1}^\infty\bW_j-I\biggr\|\le \sum_{j=i+1}^\infty \|\bP_j\|+\sum_{l=2}^{\infty}\sum_{i+1\le j_1<j_2<\cdots <j_l} \prod_{k=1}^l \|\bP_{j_k}\|\le \biggl(\sum_{j=i+1}^\infty \|\bP_j\|\biggr)\exp\biggl(\sum_{j=i+1}^\infty \|\bP_j\|\biggr) .
$$
Therefore, for big enough $i$,
$$
\biggl\|\prod_{j=i+1}^\infty\bW_j-I\biggr\|<\frac12.
$$
By a classical result from function analysis (\cite{Lax}, page 193), we conclude that for big enough $i$,
$$
\prod_{j=i+1}^\infty\bW_j=I+\biggl(\prod_{j=i+1}^\infty\bW_j-I\biggr)
$$
is invertiable and its inverse satisfies
$$
\biggl\|\biggl(\prod_{j=i+1}^\infty\bW_j\biggr)^{-1}\biggr\|\le \frac{1}{1-\biggl\|\prod_{j=i+1}^\infty\bW_j-I\biggr\|}\le 2.
$$
Consequently, for big enough $i$,
$$
\|\bbb_i\|=\biggl\|\biggl(\prod_{j=i+1}^\infty\bW_j\biggr)^{-1}\biggl(\prod_{j=i+1}^\infty \bW_j\biggr)\bbb_i\biggr\|\le 2\biggl\|\biggl(\prod_{j=i+1}^\infty \bW_j\biggr)\bbb_i\biggr\|,
$$
which together with (\ref{proplimit2eq4}) ensures the validity of equation  \eqref{proplimit2eq2*}.
\end{proof}

The necessary conditions given in Theorem \ref{proplimit2} for a ReLU network to converge provide mathematical guidelines for further construction of deep ReLU networks.

Our next task is to establish a useful sufficient condition guaranteeing the existence of limit (\ref{limit2}).

\begin{theorem}\label{theoremonlimit2}
Let $\|\cdot\|$ be a norm on $\bR^m$ that satisfies (\ref{nondecreasingvectornorm})  and $\|\cdot\|$ be its induced matrix norm. If
\begin{equation}\label{limit2sufficient3}
\sum_{n=1}^\infty \|\bbb_n\|<+\infty,
\end{equation}
\begin{equation}\label{limit2sufficient1}
\prod_{j=i}^\infty I_j\bW_j\mbox{ converges for every }i\ge2,
\end{equation}
and there exists a positive constant $C$ such that
\begin{equation}\label{limit2sufficient2}
\prod_{j=i}^n \|\bW_j\|\le C\mbox{ for all }2\le i\le n<+\infty,
\end{equation}
then the limit (\ref{limit2}) exists.
\end{theorem}
\begin{proof}
It suffices to show that
$$
\bc_n:=\sum_{i=1}^n\biggl(\prod_{j=i+1}^n I_j\bW_j\biggr)I_i\bbb_i, \ \ n\in\bN
$$
forms a Cauchy sequence in $\bR^m$. Let $\varepsilon>0$ be arbitrary. By condition (\ref{limit2sufficient3}), there exists some $p\in\bN$ such that
\begin{equation}\label{sumofbbbi}
    \sum_{i=p+1}^\infty \|\bbb_i\|<\varepsilon.
\end{equation}
According to hypothesis (\ref{limit2sufficient1}), when $n'>n$ are big enough, it holds for all $i=1,2, \dots, p$ that
\begin{equation}\label{convergenceCondition}
\biggl\|\prod_{j=i+1}^{n'}I_j\bW_j-\prod_{j=i+1}^{n}I_j\bW_j\biggr\|\le \varepsilon.
\end{equation}
For such $n'>n>p$, we estimate $\|\bc_{n'}-\bc_{n}\|$. To this end, we let
$$
\bd_{n',n,p}:=\sum_{i=1}^p\biggl(\prod_{j=i+1}^{n'}I_j\bW_j-\prod_{j=i+1}^{n}I_j\bW_j\biggr)I_i\bbb_i.
$$
Then, it follows from condition (\ref{convergenceCondition}) that for big enough $n'>n$,
\begin{equation}\label{dn'np}
\|\bd_{n',n,p}\|\le \sum_{i=1}^p\biggl\|\prod_{j=i+1}^{n'}I_j\bW_j-\prod_{j=i+1}^{n}I_j\bW_j\biggr\|\|\bbb_i\|\le \varepsilon\sum_{i=1}^p\|\bbb_i\|.
\end{equation}
Note that
\begin{equation}\label{differencecn-cn'}
\bc_{n'}-\bc_{n}=\bd_{n',n,p}+\sum_{i=p+1}^{n'}\biggl(\prod_{j=i+1}^{n'}I_j\bW_j\biggr)I_i\bbb_i+\sum_{i=p+1}^{n}\biggl(\prod_{j=i+1}^{n}I_j\bW_j\biggr)I_i\bbb_i.
\end{equation}
Employing \eqref{dn'np}, (\ref{matrixcon2}), (\ref{limit2sufficient2}), and \eqref{sumofbbbi}, we have for  big enough $n'>n$ that
$$
\begin{aligned}
\|\bc_{n'}-\bc_{n}\|
&\le \|\bd_{n',n,p}\|+\sum_{i=p+1}^{n'}\biggl(\prod_{j=i+1}^{n'}\|\bW_j\|\biggr)\|\bbb_i\|+\sum_{i=p+1}^{n}\biggl(\prod_{j=i+1}^{n}\|\bW_j\|\biggr)\|\bbb_i\|\\
&\le \varepsilon\sum_{i=1}^p\|\bbb_i\|+2C\sum_{i=p+1}^{\infty}\|\bbb_i\|\\
&\le \varepsilon\biggl(\sum_{i=1}^p\|\bbb_i\|+2C\biggr).
\end{aligned}
$$
This shows that $\bc_n$ is a Cauchy sequence and thus it converges.
\end{proof}

We remark that when $\bW_i,I_i$ all equal the identity matrix, limit (\ref{limit2}) becomes
$$
\lim_{n\to\infty}\sum_{i=1}^n\bbb_i.
$$
Thus, condition (\ref{limit2sufficient3}) is almost necessary for the existence of limit (\ref{limit2}). The other two conditions (\ref{limit2sufficient1}) and (\ref{limit2sufficient2}) are weaker than condition (\ref{sufficient2}), as explained in the proof of Theorem \ref{finalsufficient} to be presented in the next section.

\section{Sufficient Conditions for Convergence of ReLU Networks}
\setcounter{equation}{0}

In this section, we present sufficient conditions for deep ReLU networks to converge pointwise by using results established in the previous two sections. Moreover, we demonstrate that these sufficient conditions provide mathematical interpretation to the well-known deep residual networks which have achieved remarkable success in image classification.


We now establish sufficient conditions on their weight matrices and bias vectors for deep ReLU networks to converge pointwise.

\begin{theorem}\label{finalsufficient}
Let $\|\cdot\|$ be a norm on $\bR^m$ that satisfies (\ref{nondecreasingvectornorm})  and $\|\cdot\|$ be its induced matrix norm. If the weight matrices $\bW_n$, $n\ge 2$, satisfy
\begin{equation}\label{finalsufficientcon1}
   \bW_n=I+\bP_n, \ \  n\ge 2, \ \ \sum_{n=2}^\infty\|\bP_n\|<+\infty
\end{equation}
and the bias vectors $\bbb_i$, $i\in\bN$, satisfy
$$
\sum_{n=1}^\infty \|\bbb_n\|<+\infty,
$$
then the ReLU neural networks $\cN_n$ converge pointwise on $[0,1]^d$.
\end{theorem}
\begin{proof}
According to Theorem \ref{convergenceRELU}, it suffices to show under the given conditions of this theorem, limits  (\ref{limit1}) and  (\ref{limit2}) exist for all $I_n\in\cD_m$, $n\in\bN$.

Since the vector norm on $\bR^m$ satisfies (\ref{nondecreasingvectornorm}), its induced matrix norm satisfies conditions (\ref{matrixcon1}) and (\ref{matrixcon2}). By Theorem \ref{sufficientthm}, condition \eqref{finalsufficientcon1} ensures that limit (\ref{limit1}) exists for all $I_n\in\cD_m$, $n\in\bN$.

It remains to confirm that limit (\ref{limit2}) exists for all $I_n\in\cD_m$, $n\in\bN$. By the proof of Theorem \ref{sufficientthm}, condition (\ref{limit2sufficient1}) is satisfied when condition (\ref{finalsufficientcon1}) is fulfilled. We can also verify by using properties of the exponential function that
$$
\prod_{j=i}^n \|\bW_j\|\le\prod_{j=i}^n (1+\|\bP_j\|)\le \prod_{j=i}^n\exp(\|\bP_j\|)\le \exp\biggl(\sum_{j=2}^\infty\|\bP_j\|\biggr),\ \ 2\le i\le n<+\infty.
$$
Therefore, condition (\ref{limit2sufficient2}) is also satisfied with the constant
$$
C:=\exp\biggl(\sum_{j=2}^\infty\|\bP_j\|\biggr).
$$
By Theorem \ref{theoremonlimit2}, limit (\ref{limit2}) exists for all $I_n\in\cD_m$, $n\in\bN$.

Finally, by part 2 of Theorem \ref{convergenceRELU}, we conclude that the ReLU deep network $\cN_n$ converges pointwise on $[0,1]^d$ as $n$ tends to infinity.
\end{proof}

We remark that under the conditions in Theorem \ref{finalsufficient} or Proposition \ref{proplimit2}, it holds
$$
\lim_{n\to\infty} \bW_n x+\bbb_n=x,\ \ x\in\bR^m,
$$
which reveals that for deep layers, the linear function $\bW_nx+\bbb_n$ will be close to the identity mapping. Thus the deep weight layers of a ReLU network apply gradual changes to the ultimate input-output relation determined by the network. This justifies the design strategy of the successful deep Residual Networks (ResNets) for image recognition \cite{KaimingHe,KaimingHe2}.

Deep ResNets achieved an $3.57\%$ top-5 error on the
ImageNet test set, and won the 1st place in the ILSVRC
2015 classification competition. They also have excellent generalization performance
on other recognition tasks. These successes are due to the ability of ResNets to accommodate very deep layers (more than 1000 layers in some applications). This ability originates from the design strategy of ResNets.

A deep ResNet consists of many stacked residual units of the form:
$$
\bx_{l+1}=\sigma(\bx_l+\cF(\bx_l,\bW_l,\bbb_l))
$$
where $\bx_l$ and $\bx_{l+1}$ are input and output of the $l$-th unit, $\bW_l,\bbb_l$ are the weight matrices and bias vectors of the unit, and the residual function $\cF(\bx_l,\bW_l,\bbb_l)$ is a ReLU network with at least two layers. In this manner, instead of directly learning a desired underlying mapping by the stacked layers, the ResNet explicitly let these layers to fit the residual function. The design is motivated from the intuition that the residual functions will be close to zero, and thus $\bx_l+\cF(\bx_l,\bW_l,\bbb_l)$ will be close to the identity mapping for deep units. This motivation was supported by numerical experiments presented in \cite{KaimingHe}. Our results stated in Theorem \ref{finalsufficient} provide a theoretical support for the
motivation.

{\small
\bibliographystyle{amsplain}

}

\end{document}